\PassOptionsToPackage{table}{xcolor}
\documentclass[sigconf]{aamas}
\makeatletter
\def\aamas@legacy@editors{\fi}
\ifx\acmConference@editors\aamas@legacy@editors
  \expandafter\@firstofone
\else
  \expandafter\@gobble
\fi
{\acmConference@editors\gdef\acmConference@editors{}}
\let\aamas@legacy@editors\@undefined
\makeatother

\usepackage[T1]{fontenc}
\usepackage[utf8]{inputenc}
\usepackage{microtype}
\usepackage{graphicx}
\usepackage{booktabs}
\usepackage{amsmath,amsthm}
\usepackage{mathtools}
\usepackage{bm}
\usepackage[table]{xcolor}
\usepackage{tabularx}
\newcommand{\cmark}{\ding{51}}
\newcommand{\xmark}{\ding{55}}
\usepackage{array}
\usepackage{tabularx}
\usepackage{enumitem}
\usepackage{cleveref}
\usepackage{multirow}
\usepackage{xcolor}
\numberwithin{equation}{section}

\AtEndPreamble{%
  \hypersetup{
    pdfauthor={},
    colorlinks=true,
    linkcolor=blue,
    citecolor=black,
    urlcolor=blue,
    pdfborder={0 0 0},
    pdfborderstyle={}
  }%
}

\newsavebox{\ManuscriptLinkBox}
\DeclareRobustCommand{\ManuscriptLinkFrame}[2]{%
  \begingroup
    \setlength{\fboxsep}{0.3pt}%
    \setlength{\fboxrule}{0.6pt}%
    \sbox{\ManuscriptLinkBox}{#2}%
    \leavevmode
    \rlap{%
      \kern-\dimexpr\fboxsep+\fboxrule\relax
      \smash{\textcolor{#1}{\fbox{\phantom{\usebox{\ManuscriptLinkBox}}}}}%
    }%
    \usebox{\ManuscriptLinkBox}%
  \endgroup
}
\AddToHook{begindocument/end}{%
  \let\ManuscriptOriginalCiteNumFont\citenumfont
  \renewcommand*{\citenumfont}[1]{%
    \ManuscriptLinkFrame{green}{\ManuscriptOriginalCiteNumFont{#1}}%
  }%
  \let\ManuscriptOriginalEqref\eqref
  \renewcommand*{\eqref}[1]{%
    \begingroup
      \hypersetup{linkcolor=black}%
      \ManuscriptLinkFrame{red}{\ManuscriptOriginalEqref{#1}}%
    \endgroup
  }%
}

\AtBeginEnvironment{figure*}{%
  \ifcase\value{figure}
    \Description{Six-stage framework: local modules identify objects, propose event matches, and match intervention responses. An aggregator validates interfaces, composes local mechanisms for downstream prediction, and optionally guides local actions with execution feedback.}%
  \or
    \Description{Causal composition routes outgoing events through validated interfaces into destination mechanisms. The optional local controller checks prerequisites, predicts effects, gates actions, and compares observed and predicted outcomes, with revision and execution feedback.}%
  \fi
}

\makeatletter
\gdef\@copyrightpermission{
  \begin{minipage}{0.2\columnwidth}
   \href{https://creativecommons.org/licenses/by/4.0/}{\includegraphics[width=0.90\textwidth]{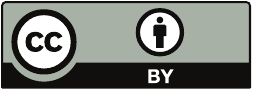}}
  \end{minipage}\hfill
  \begin{minipage}{0.8\columnwidth}
   \href{https://creativecommons.org/licenses/by/4.0/}{This work is licensed under a Creative Commons Attribution International 4.0 License.}
  \end{minipage}
  
}
\makeatother
\setcopyright{ifaamas}
\acmConference[AAMAS 2027]{Proc.\@ of the 26th International Conference
on Autonomous Agents and Multiagent Systems (AAMAS 2027)}{3--7 May 2027}{Hanoi, Vietnam}{M.~Baldoni, F.~Fang, W.~Yeoh, N.~Yorke-Smith (eds.)}
\copyrightyear{2027}
\acmYear{2027}
\acmDOI{}
\acmISBN{}
\acmSubmissionID{}
\submissionType{Research Paper Track}

\newtheorem{theorem}{Theorem}

\newcommand{\fcc}{\textnormal{\textsc{CausalCompose}}}
\newcommand{\E}{\mathbb{E}}
\newcommand{\cM}{\mathcal{M}}

\newcommand{\Pobs}{P_{\mathrm{obs}}}
\newcommand{\Pint}{P_{\mathrm{int}}}

\newcommand{\Fobs}{\hat{F}^{\mathrm{obs}}}

\newcommand{\seqwm}{\textsc{Seq-WM}}
\newcommand{\discwm}{\textsc{Discovered-Causal}}
\newcommand{\oraclewm}{\textsc{Oracle-Causal}}

\crefname{assumption}{Assumption}{Assumptions}
\Crefname{assumption}{Assumption}{Assumptions}
\crefrangeformat{assumption}{Assumptions~#3#1#4--#5#2#6}

\title{When Do Causal World Models Help Modular LLM Agents}
\author{Xinyuan Song}
\affiliation{
  \institution{Emory University}
  \city{Atlanta}
  \country{USA}
}
\email{xinyuan.song@emory.edu}
\author{Zekun Cai}
\affiliation{
  \institution{The University of Tokyo}
  \city{Tokyo}
  \country{Japan}
}
\email{caizekun@csis.u-tokyo.ac.jp}

\begin{abstract}
LLM agents increasingly operate in modular systems, where planning requires coordinating multiple functional components whose actions can causally change downstream states and behaviors. This makes intervention-aware prediction essential: an agent must anticipate not only what is likely to happen next, but what will happen as a consequence of its own actions. We find that purely observational world models can incur irreducible intervention-time error in such systems, even when they accurately model observational trajectories. To address this limitation, we propose \fcc, a causal world-model framework that equips modular LLM agents with causal structure for intervention-aware prediction and planning. We theoretically show that causal composition can overcome this intervention gap with sufficient evidence and controlled local error. Experiments across diverse modular agent environments validate these results: incorporating causal structure consistently improves intervention-sensitive reasoning and planning, with particularly clear gains in structured tool-use systems where causal dependencies directly govern action execution. Our results show when causal world models help modular LLM agents: when causal dependencies are identifiable and actionable for planning.
\end{abstract}


\begin{document}
\emergencystretch 3em
\maketitle

\section{Introduction}

LLM agents increasingly operate in systems that are modular by design~\cite{liao2023apibank,lu2024toolsandbox,yao2024taubench,barres2025tau2bench}. Rather than interacting with a single monolithic environment, an agent often coordinates multiple functional modules~\cite{liao2023apibank,lu2024toolsandbox}. For example, a retail agent must coordinate payment and fulfillment services to complete an order~\cite{yao2022webshop,yao2024taubench,barres2025tau2bench}, while a household agent must coordinate navigation and object interaction to accomplish a physical task~\cite{shridhar2020alfworld}. Modern agent benchmarks make this modular structure increasingly explicit: APIBank~\cite{liao2023apibank} organizes interaction around executable APIs, ToolSandbox~\cite{lu2024toolsandbox} evaluates stateful tool use, $\tau$-bench~\cite{yao2024taubench} and $\tau^2$-bench~\cite{barres2025tau2bench} couple language interaction with domain tools and evolving environment state, and ALFWorld~\cite{shridhar2020alfworld} requires agents to coordinate distinct functional operations over a shared embodied task. This modular organization enables agents to decompose complex tasks into specialized components, reuse functional capabilities across tasks, and coordinate heterogeneous tools and services within a unified workflow. A modular design therefore provides a scalable and flexible foundation for building agents that operate across increasingly complex environments.

This modularity creates a causal problem. Sequence-based world models learn which states and actions tend to occur together, whereas planning requires predicting what will happen because the agent takes a particular action. This distinction becomes fundamental when modules share latent causes or when several components change together: an observationally predictive dependency need not remain valid after intervention \citep{pearl2009causality,peters2016causal,scholkopf2021causalrep}. For example, an agent may repeatedly observe a payment update followed by a shipment transition and learn this sequence accurately, while still predicting the wrong downstream state when it actively changes the payment module. More generally, model-based RL and language-agent world models are commonly trained to predict future observations or trajectories from recorded experience \citep{janner2019mbpo,hafner2018planet,hafner2019dreamer,hafner2020dreamerv2,hafner2023dreamerv3,wang2024llmworldsim,rwom2025,dynathink2025,text2world2025}. In modular systems, however, accurate observational prediction alone does not guarantee correct intervention-time prediction.

\begin{figure*}[!ht]
  \centering
  \includegraphics[width=0.75\linewidth]{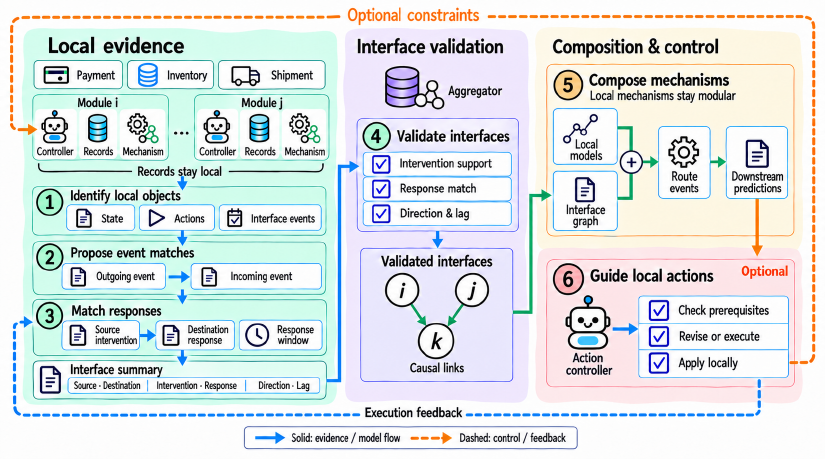}
    
\caption{\textbf{Overview of \fcc.}
\textbf{(a) Local evidence:} each module preserves local records and mechanisms while exposing compact intervention-response summaries.
\textbf{(b) Interface validation:} aggregated evidence validates causal dependencies, including direction and lag, to construct the causal interface graph.
\textbf{(c) Composition and control:} the validated graph is composed with local mechanisms for downstream prediction and action guidance, with execution feedback supporting subsequent updates.}
  \label{fig:pipeline}
  
\end{figure*}

We first formalize why observational world models can fail at intervention time. Even if a world model perfectly learns the predictive relationships present in observational data, those relationships need not correctly predict the consequences of actively changing one module~\cite{pearl2009causality,peters2016causal}. We characterize this mismatch through unblocked back-door paths~\cite{pearl2009causality} and prove that, under nondegenerate confounding, the optimal observational predictor still incurs a strictly positive error gap relative to the correct interventional predictor. Thus, collecting more observational data or fitting the observational distribution more accurately cannot in general eliminate the error: resolving it requires information that distinguishes observation from intervention~\cite{pearl2009causality,peters2016causal}. We further show that this problem becomes more consequential in modular systems, because an incorrect intervention prediction at one module can propagate through subsequent module dependencies, causing prediction reliability to deteriorate as the composition depth increases.

To address this problem, we propose \fcc, a causal world model for modular LLM agents. Instead of treating the system as one undifferentiated trajectory, \fcc preserves local mechanisms and uses local actions and responses to identify causal dependencies among functional modules. As illustrated in Figure~\ref{fig:pipeline}, \fcc first extracts intervention-response evidence locally, validates the resulting causal dependencies, and then composes the validated structure with local mechanisms for downstream prediction and optional action guidance. This composition turns observational modular dynamics into an intervention-aware world model for planning. We theoretically show that causal recovery improves with intervention-response coverage and that, with sufficient coverage and controlled local mechanism error, causal composition can achieve lower intervention-time error than an observational world model.

We evaluate these predictions through synthetic causal diagnostics and modular LLM-agent environments~\cite{shridhar2020alfworld,liao2023apibank,lu2024toolsandbox,yao2024taubench,barres2025tau2bench}. The synthetic experiments verify the predicted intervention gap, horizon compounding, and dependence of causal recovery on intervention coverage. Agent experiments further show that incorporating causal structure improves intervention-sensitive reasoning and planning, with the clearest gains in structured tool-use environments~\cite{liao2023apibank,lu2024toolsandbox,yao2024taubench,barres2025tau2bench} where modular dependencies directly govern valid actions and downstream transitions. Across dialogue and narrative settings, the experiments further demonstrate that causal knowledge becomes useful when it is presented in a form that makes the relevant dependency actionable at decision time. Together, these results answer the question posed by our title: causal world models help modular LLM agents when intervention-sensitive dependencies can be identified from available evidence and incorporated into agent planning.

Our contributions are summarized as follows:
\begin{enumerate}[leftmargin=1.5em, itemsep=1pt, topsep=2pt, parsep=0pt, partopsep=0pt]
    \item We identify a fundamental mismatch between observational world modeling and intervention-time planning in modular LLM agents.
    \item We establish the intervention-time error of observational models, characterize its compounding across modules, and derive conditions under which causal composition overcomes this gap.
    \item We propose \fcc, which learns causal dependencies from intervention-response evidence and composes local mechanisms for prediction and planning.
    \item We validate our theory across synthetic diagnostics and modular agent environments, showing when causal world models improve agent planning.
\end{enumerate}

\section{Related Work}

\paragraph{Federated and causal learning.}
Federated RL~\citep{momentum_fedrl_2024} shares policies, value functions, or learned models across decentralized clients, while model-based RL learns dynamics for planning and control~\citep{janner2019mbpo,hafner2018planet,hafner2019dreamer,hafner2020dreamerv2,hafner2023dreamerv3}. Causal discovery methods learn structure from observational data~\citep{zheng2018notears}, interventions or heterogeneous environments~\citep{peters2016causal,brouillard2020dcdi,perry2022causaldiscoveryheterogeneousenvironments}, and distributed data~\citep{mian2023nothing,baldo2026regretbasedfederatedcausaldiscovery,fedci2026}. Unlike settings where clients represent separate environments, \fcc\ treats clients as functional modules of one coupled system and uses intervention--response evidence to recover and compose dependencies between their local mechanisms.

\paragraph{Causal and language-agent world models.}
Causal representation learning and causal world models study how causal structure supports generalization, prediction, and control~\citep{scholkopf2021causalrep,xrl_causal_wm_2023,causal_mbrl_provably_2022}. Recent LLM-agent work further explores environment simulation, retrieval, dynamic reasoning, adaptive world models, state tracking, and action-conditioned reasoning~\citep{wang2024llmworldsim,rwom2025,dynathink2025,text2world2025,patchworld2026,llm_state_2023,yao2023react}. We instead ask when causal dependencies between functional modules improve intervention-time prediction and planning.

\paragraph{Agent benchmarks and modular systems.}
Interactive benchmarks evaluate agents over tool use, service workflows, embodied environments, web interaction, and device control~\citep{yao2024taubench,barres2025tau2bench,liao2023apibank,lu2024toolsandbox,yao2022webshop,shridhar2020alfworld,wang2022scienceworld,rawles2024androidworld}. We use their naturally separable functional components to make cross-module dependencies explicit and evaluate when modeling these dependencies causally improves prediction and planning.

\paragraph{Relationship to prior work.}
Table~\ref{tab:related-work-novelty} compares \fcc\ with closely related work in causal world modeling, modular structure, intervention-based evidence, cross-module composition, planning-time use, and theoretical characterization of when causal modeling improves prediction.

\begin{table*}[!ht]
\centering
\caption{\textbf{Comparison with closely related work.}
We compare methods by whether they support causal world modeling, modular
structure, intervention-based evidence, cross-module causal composition,
planning-time use, and theoretical characterization.}
\label{tab:related-work-novelty}

\small
\setlength{\tabcolsep}{2.5pt}
\renewcommand{\arraystretch}{1.12}

\resizebox{0.8\textwidth}{!}{%
\begin{tabular}{
@{}
>{\raggedright\arraybackslash}p{0.30\textwidth}
>{\centering\arraybackslash}p{0.10\textwidth}
>{\centering\arraybackslash}p{0.10\textwidth}
>{\centering\arraybackslash}p{0.12\textwidth}
>{\centering\arraybackslash}p{0.13\textwidth}
>{\centering\arraybackslash}p{0.10\textwidth}
>{\centering\arraybackslash}p{0.10\textwidth}
@{}
}
\toprule
\textbf{Work}
& \textbf{Causal WM}
& \textbf{Modular}
& \textbf{Intervention evidence}
& \textbf{Cross-module composition}
& \textbf{Planning use}
& \textbf{Theory}
\\
\midrule

Gkountouras et al.~\citep{gkountouras2025language}
& \cmark & \xmark & \cmark & \xmark & \cmark & \xmark \\

CausalPlan~\citep{nguyen2025causalplan}
& \cmark & \xmark & \cmark & \xmark & \cmark & \xmark \\

Dyna-Think~\citep{dynathink2025}
& \xmark & \cmark & \xmark & \cmark & \cmark & \xmark \\

PatchWorld~\citep{patchworld2026}
& \xmark & \cmark & \xmark & \cmark & \cmark & \xmark \\

\midrule

WM3C~\citep{wang2025modeling}
& \cmark & \cmark & \cmark & \cmark & \cmark & \xmark \\

PoE-World~\citep{piriyakulkij2025poeworld}
& \xmark & \cmark & \xmark & \cmark & \cmark & \xmark \\

C-MBPO~\citep{causal_mbpo_2025}
& \cmark & \cmark & \cmark & \cmark & \cmark & \cmark \\

\midrule

I-PERI~\citep{baldo2026regretbasedfederatedcausaldiscovery}
& \cmark & \cmark & \cmark & \xmark & \xmark & \cmark \\

fedCI-IOD~\citep{fedci2026}
& \cmark & \cmark & \xmark & \xmark & \xmark & \cmark \\

\midrule
\rowcolor{gray!15}
\textbf{\fcc\ (This work)}
& \cmark & \cmark & \cmark & \cmark & \cmark & \cmark \\
\bottomrule

\end{tabular}%
}

\end{table*}

\section{Why Non-Causal World Models Fail}
\label{sec:noncausal}

Non-causal world models can fail because predicting what follows an action is not equivalent to predicting its causal effect. In modular agents, hidden factors may influence both action selection and downstream outcomes, causing observational models to learn incorrect interventional dynamics despite accurate observational prediction. We formalize this mismatch and its resulting rollout errors below.

\paragraph{Modular causal world model.}
Consider a modular agent with $K$ interacting modules. Let $X_{k,t}$ denote the state of module $k$ at time $t$, and let $A_{k,t}$ denote the action executed by that module. We model the local transition as
\begin{equation}
X_{k,t+1}=F_k^*(X_{k,t},A_{k,t},X_{\mathcal{P}_k,t},U_{k,t}),
\label{eq:body-local}
\end{equation}
where $\mathcal{P}_k$ denotes the set of modules with direct causal influence on module $k$, and $U_{k,t}$ collects unobserved factors. These dependencies define a directed causal graph
\begin{equation}
G^*=(\cM,E^*),
\label{eq:body-causal-graph}
\end{equation}
where $(i,k)\in E^*$ indicates that module $i$ has a direct causal influence on module $k$. Equation~\eqref{eq:body-causal-graph} therefore connects the local mechanisms in Equation~\eqref{eq:body-local} and specifies how causal effects can propagate across modules.

This representation makes the limitation of a non-causal world model explicit. Given the observed system state $X_t=X_{1:K,t}$, an observational model learns transitions of the form
\begin{equation}
\Pobs(X_{j,t+1}\mid X_t,A_{k,t}=a),
\label{eq:body-observational}
\end{equation}
whereas an agent that actively executes action $a$ requires the interventional transition
\begin{equation}
\Pint(X_{j,t+1}\mid X_t,\do(A_{k,t}=a)).
\label{eq:body-interventional}
\end{equation}
Although Equations~\eqref{eq:body-observational} and~\eqref{eq:body-interventional} appear similar, they can differ when unobserved factors jointly affect action selection and downstream outcomes. In particular, suppose an unobserved variable $U$ influences both the action $A_{k,t}$ and the downstream state $X_{j,t+1}$:
\begin{equation}
A_{k,t}\leftarrow U\rightarrow X_{j,t+1}.
\label{eq:body-confounding}
\end{equation}
Because $U$ is not included in the observed system state $X_t$, conditioning on $X_t$ and $A_{k,t}$ can mix the effect of the action with the influence of $U$. As a result, the observational transition in Equation~\eqref{eq:body-observational} does not in general recover the interventional transition in Equation~\eqref{eq:body-interventional}, and for some actions they may differ:
\begin{equation}
\Pobs(X_{j,t+1}\mid X_t,A_{k,t}=a)\neq\Pint(X_{j,t+1}\mid X_t,\do(A_{k,t}=a)).
\label{eq:body-causal-gap}
\end{equation}
Equation~\eqref{eq:body-causal-gap} shows why observational accuracy is insufficient for an acting agent. A causal world model instead predicts the interventional transition in Equation~\eqref{eq:body-interventional} by propagating the effect of $\do(A_{k,t}=a)$ through the causal graph in Equation~\eqref{eq:body-causal-graph}.

\paragraph{Interventional prediction gap.}
The distinction in Equation~\eqref{eq:body-causal-gap} can be made explicit in the linear-Gaussian Structural Causal Model (SCM). Suppose an unobserved variable enters the downstream mechanism of module $j$ through $V_j$ with strength $\gamma$, and intervening on $A_{k,t}$ changes its conditional mean by $\mu_{\mathrm{shift}}$. The resulting displacement between the observational and interventional predictions is
\begin{equation}
\Delta_{\mathrm{int}}=\gamma V_j\mu_{\mathrm{shift}}.
\label{eq:body-linear-gap}
\end{equation}
Equation~\eqref{eq:body-linear-gap} is nonzero whenever $\gamma\neq 0$ and $V_j\mu_{\mathrm{shift}}\neq 0$, providing a concrete instance of the observational--interventional mismatch in Equation~\eqref{eq:body-causal-gap}.

\begin{theorem}[Interventional risk gap]
\label{thm:body-risk}
Let $\Fobs$ be the population-optimal squared-error predictor under the observational transition in Equation~\eqref{eq:body-observational}. Under the intervention $\do(A_{k,t}=a)$, define its interventional risk as
\begin{equation}
R_{\mathrm{int}}(\Fobs):=\E_{\Pint}\left[\left\|X_{j,t+1}-\Fobs(X_t,a)\right\|_2^2\right].
\label{eq:body-int-risk}
\end{equation}
Then, this risk satisfies
\begin{equation}
R_{\mathrm{int}}(\Fobs)\ge\E_{\Pint}\left[\left\|\E_{\Pint}[X_{j,t+1}\mid X_t,\do(A_{k,t}=a)]-\Fobs(X_t,a)\right\|_2^2\right].
\label{eq:body-risk-gap}
\end{equation}
For the linear-Gaussian setting in Equation~\eqref{eq:body-linear-gap}, Equation~\eqref{eq:body-risk-gap} further gives
\begin{equation}
R_{\mathrm{int}}(\Fobs)\ge\gamma^2\left\|V_j\mu_{\mathrm{shift}}\right\|_2^2.
\label{eq:body-linear-risk}
\end{equation}
\end{theorem}

Theorem~\ref{thm:body-risk} shows that observational accuracy alone cannot remove the intervention-time error induced by Equation~\eqref{eq:body-causal-gap}. A causal world model instead targets the interventional transition in Equation~\eqref{eq:body-interventional} and uses the causal graph in Equation~\eqref{eq:body-causal-graph} to propagate the effect of an action to downstream modules. The complete proof of Theorem~\ref{thm:body-risk} is provided in Appendix~\ref{app:proof-risk-gap}.

\paragraph{Why the causal model can win.}
In a modular system, an observational--interventional mismatch can propagate through multiple causal paths. For the linear-Gaussian setting, let $B_e$ denote the causal operator associated with edge $e$, and let $\Pi(k,j)$ denote the set of directed paths from the intervened module $k$ to a downstream module $j$. The total causal effect is
\begin{equation}
T_{j\leftarrow k}:=\sum_{\pi\in\Pi(k,j)}\prod_{e\in\pi}B_e.
\label{eq:body-total-effect}
\end{equation}
If the intervention at module $k$ induces the mean displacement $\gamma V_k\mu_{\mathrm{shift}}$, this discrepancy propagates to module $j$ as
\begin{equation}
b_j=T_{j\leftarrow k}\gamma V_k\mu_{\mathrm{shift}}.
\label{eq:body-propagated-bias}
\end{equation}
Thus, while the observational predictor propagates a systematically biased state through $G^*$, the causal predictor propagates the interventional conditional mean.

\begin{theorem}[Causal advantage under propagation]
\label{thm:body-causal-advantage}
For any downstream module $j$, the intervention-time excess risk of the population-optimal observational predictor satisfies
\begin{equation}
\begin{aligned}
R_{\mathrm{int}}(\Fobs;j)-R_{\mathrm{int}}(F_{\mathrm{causal}};j)
&=\|b_j\|_2^2\\
&=\gamma^2\left\|T_{j\leftarrow k}V_k\mu_{\mathrm{shift}}\right\|_2^2.
\end{aligned}
\label{eq:body-propagated-risk}
\end{equation}
Consequently,
\begin{equation}
\begin{aligned}
\gamma^2\sigma_{\min}^2(T_{j\leftarrow k})
\left\|V_k\mu_{\mathrm{shift}}\right\|_2^2
&\le R_{\mathrm{int}}(\Fobs;j)-R_{\mathrm{int}}(F_{\mathrm{causal}};j)\\
&\le\gamma^2\|T_{j\leftarrow k}\|_2^2
\left\|V_k\mu_{\mathrm{shift}}\right\|_2^2.
\end{aligned}
\label{eq:body-propagated-bounds}
\end{equation}
\end{theorem}

Theorem~\ref{thm:body-causal-advantage} shows that causal advantage depends on both the initial confounding bias and its propagation through the causal graph. While $\gamma V_k\mu_{\mathrm{shift}}$ determines the local bias, $T_{j\leftarrow k}$ aggregates all directed paths to the downstream module, allowing the bias to be amplified, attenuated, or canceled across paths. Thus, identical local confounding can produce different downstream errors under different causal structures. The proof of Theorem~\ref{thm:body-causal-advantage} is provided in Appendix~\ref{app:proof-causal-advantage}.

\begin{figure*}[!ht]
  \centering
  \includegraphics[width=0.75\linewidth]{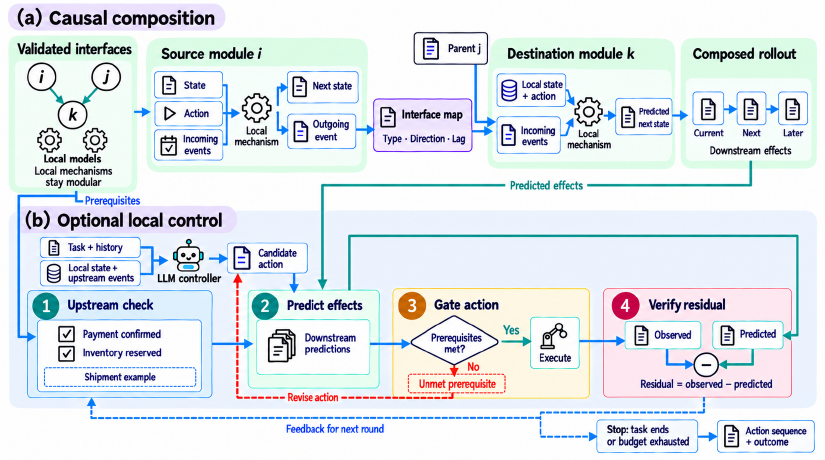}
    
  \caption{\textbf{Causal composition and local control in \fcc.}
  \textbf{(a)} Validated interfaces connect local mechanisms to propagate intervention effects and predict downstream states.
  \textbf{(b)} The composed model can optionally guide local actions through prerequisite checks, effect prediction, action gating, and post-action verification.}
  \label{fig:composition_control}
    
\end{figure*}

%
%

\section{CausalCompose}
\label{sec:method}

\fcc\ turns distributed interaction traces into an \\intervention-aware world model for modular LLM agents. Rather than fitting a single global transition model, it preserves local mechanisms and learns the causal interfaces needed to compose them. Given $K$ modules with local states $X_{k,t}$, actions $A_{k,t}$, and transition mechanisms defined in Equation~\eqref{eq:body-local}, \fcc\ identifies and validates cross-module intervention--response dependencies, then composes them for downstream prediction and action selection.

Figure~\ref{fig:pipeline} summarizes the pipeline: extracting local intervention--response evidence, validating causal interfaces, and composing them for prediction and optional control. Modules exchange only compact interface evidence---including events, responses, direction, and lag---rather than complete trajectories, preserving modularity while capturing cross-module causal propagation.

\subsection{Local Intervention--Response Evidence}
\label{sec:local-evidence}

The first stage converts local execution traces into candidate causal evidence. As illustrated in Figure~\ref{fig:pipeline}, each module records its local state, executed actions, resulting state changes, and externally visible interface events. Following the local dynamics in Equation~\eqref{eq:body-local}, \fcc\ treats an executed action $A_{i,t}=a$ as intervention-like evidence and records the corresponding change in $X_{i,t}$ together with outgoing events that may affect other modules.

Instead of sharing complete trajectories, modules exchange compact records of interventions, interface events, downstream responses, and timing information. \fcc\ matches outgoing events from source modules with compatible incoming events and state changes in destination modules, providing empirical evidence for the interventional dependencies defined by $P_{\mathrm{int}}(\cdot\mid\do(A_{i,t}=a))$ in Equation~\eqref{eq:body-interventional}.

Because temporal matches alone do not establish causality, \fcc\ checks repeated executions for consistent intervention--response patterns, including response direction and lag. The resulting interface summaries identify candidate cross-module dependencies and provide the evidence used for causal validation in the next stage.

\subsection{Causal Interface Validation}
\label{sec:interface-validation}

Candidate interfaces may arise from incidental temporal co-occurrence and therefore must be validated before they are used for planning. Figure~\ref{fig:pipeline}(b) shows this aggregation step. For each proposed edge $(i,j)$, \fcc\ collects evidence across repeated intervention opportunities and evaluates three properties: whether the source intervention is sufficiently supported, whether a compatible destination response repeatedly follows it, and whether the direction and lag are consistent with the proposed causal interface.

We represent the result of verification trial $n$ by
\begin{equation}
M_{ij}^{(n)}\sim\mathrm{Bernoulli}(p_{ij}),
\qquad
p_{ij}=q_i r_{ij},
\label{eq:fcc-match}
\end{equation}
where $q_i$ is the probability that the relevant intervention on module $i$ is realized and $r_{ij}$ is the probability that its downstream response at module $j$ is detected. Given $N_{ij}$ verification opportunities, the interface is retained when
\begin{equation}
(i,j)\in\widehat E
\quad\text{if}\quad
\sum_{n=1}^{N_{ij}}M_{ij}^{(n)}\ge1.
\label{eq:fcc-edge-rule}
\end{equation}
The recovered parent set is therefore
\begin{equation}
\widehat{\mathcal P}_j
=
\{i:(i,j)\in\widehat E\},
\end{equation}
and the collection of validated interfaces forms the recovered causal graph
\begin{equation}
\widehat G=(\mathcal M,\widehat E).
\label{eq:fcc-recovered-graph}
\end{equation}

This design separates \emph{local mechanism learning} from \emph{cross-module structure learning}. The former determines how each module responds locally, while the latter determines which mechanisms must be connected when predicting the consequence of an intervention. Consequently, a new interface can be added or removed without retraining an undifferentiated global world model.

For each true edge, we assume conditionally independent verification trials and uniformly positive intervention and detection probabilities,
\begin{equation}
q_i\ge q_{\min}>0,\qquad
r_{ij}\ge r_{\min}>0,\qquad
(i,j)\in E^*.
\label{eq:fcc-min-prob}
\end{equation}
We further define
\begin{equation}
N_{\min}:=
\min_{(i,j)\in E^*}N_{ij}.
\label{eq:fcc-nmin}
\end{equation}

\begin{theorem}[Coverage of true causal edges]
\label{thm:body-coverage}
Suppose the verification model in Equation~\eqref{eq:fcc-match} holds, and the trials for each true edge are conditionally independent. Under the recovery rule in Equation~\eqref{eq:fcc-edge-rule} and the positive coverage conditions in Equation~\eqref{eq:fcc-min-prob},
\begin{equation}
\Pr(E^*\nsubseteq\widehat E)
\le
|E^*|\exp(-N_{\min}q_{\min}r_{\min}).
\label{eq:fcc-coverage}
\end{equation}
Consequently, for any $\delta\in(0,1)$,
\begin{equation}
N_{\min}
\ge
\frac{\log(|E^*|/\delta)}
{q_{\min}r_{\min}}
\label{eq:fcc-coverage-sufficient}
\end{equation}
guarantees
$\Pr(E^*\subseteq\widehat E)\ge1-\delta$.
\end{theorem}

Theorem~\ref{thm:body-coverage} shows that true-edge coverage improves exponentially with the verification budget. The rate depends on both intervention exposure and response detectability: causal interfaces cannot be reliably recovered if the relevant intervention is never executed or its downstream response cannot be observed. This directly motivates the intervention--response collection stage in Figure~\ref{fig:pipeline}.

\subsection{Intervention-Aware Causal Composition}
\label{sec:causal-composition}

Once the interfaces are validated, \fcc\ composes the local mechanisms through the recovered causal structure. As illustrated in Figure~\ref{fig:composition_control}(a), a candidate intervention is first evaluated by its local mechanism, and the resulting interface event is propagated only to causally connected modules in $\widehat G$. Each destination module then combines the incoming event with its local state and mechanism to predict subsequent states and downstream effects.

Rather than fitting a global transition function, \fcc\ recursively composes these local mechanisms along validated causal paths. In the linear setting, this propagation is characterized by the total-effect operator in Equation~\eqref{eq:body-total-effect}, which aggregates the effects transmitted along causal paths between modules. Thus, validated interfaces determine where intervention effects propagate, while local mechanisms determine how those effects transform.

The resulting rollout captures both immediate and downstream consequences of a candidate intervention. Unlike \seqwm, which predicts from observational trajectories, \fcc\ explicitly propagates interventions through the recovered causal structure.

\subsection{Optional Causal Control for Action Selection}
\label{sec:causal-control}

Beyond prediction, \fcc\ can use the composed causal model as a closed-loop controller during action selection. As shown in Figure~\ref{fig:composition_control}(b), at each step the LLM first proposes a candidate action. \fcc\ then checks its upstream prerequisites using $\widehat{\mathcal P}_k$, predicts its local and downstream consequences through the causal composition in Equation~\eqref{eq:body-total-effect}, and gates the action accordingly. Actions with missing prerequisites or conflicting predicted effects are revised; otherwise, the selected action is executed.

Execution closes the loop. After observing the resulting downstream state, \fcc\ compares it with the causal prediction. For predicted state $\widehat X_{j,t+1}$, the residual $e_{j,t+1}=X_{j,t+1}-\widehat X_{j,t+1}$ measures the prediction--execution mismatch. The observed state and residual are then fed into the next decision step, updating the evidence available for interface verification and informing the next candidate action. The control loop therefore repeats as
\emph{propose $\rightarrow$ check $\rightarrow$ predict $\rightarrow$ gate $\rightarrow$ execute $\rightarrow$ verify $\rightarrow$ feedback}
until the task terminates.

This gives \fcc\ two operational modes: \emph{prediction-only}, where causal composition does not alter the policy, and \emph{control}, where causal predictions and execution feedback repeatedly guide subsequent decisions. This distinction separates the value of recovering causal structure from the agent's ability to exploit it online.

\subsection{Error of causal composition.}
Graph recovery alone is not sufficient: the causal mechanisms used for prediction must also be accurate. Suppose each local causal mechanism has squared prediction error at most $\varepsilon$. Because local errors propagate through the same causal paths as intervention effects, their downstream contribution at module $j$ is controlled by the total-effect operators from Equation~\eqref{eq:body-total-effect}. In particular, if the relevant upstream modules are $\mathcal{P}_j$, the composition error is bounded by
\begin{equation}
R_{\mathrm{int}}(\widehat F_{\mathrm{causal}};j)-R_{\mathrm{int}}(F_{\mathrm{causal}};j)\le\varepsilon\left(\sum_{v\in\mathcal{P}_j}\|T_{j\leftarrow v}\|_2\right)^2
\label{eq:fcc-composition-error}
\end{equation}
when all required causal edges are recovered.

Theorem~\ref{thm:body-coverage} additionally bounds the probability that this condition fails. Combining the two effects yields the following sufficient condition for causal prediction to retain the advantage established in Theorem~\ref{thm:body-causal-advantage}.

\begin{theorem}[Strict causal advantage]
\label{thm:body-dominance}
With probability at least $1-|E^*|\exp(-N_{\min}q_{\min}r_{\min})$, the intervention-time risk gap satisfies
\begin{equation}
\begin{aligned}
&R_{\mathrm{int}}(\Fobs;j)-R_{\mathrm{int}}(\widehat F_{\mathrm{causal}};j)\\
&\quad\ge\gamma^2\left\|T_{j\leftarrow k}V_k\mu_{\mathrm{shift}}\right\|_2^2
-\varepsilon\left(\sum_{\substack{v\in\mathcal M\\T_{j\leftarrow v}\neq0}}\|T_{j\leftarrow v}\|_2\right)^2.
\end{aligned}
\label{eq:fcc-dominance-gap}
\end{equation}
Therefore,
\begin{equation}
\gamma^2\left\|T_{j\leftarrow k}V_k\mu_{\mathrm{shift}}\right\|_2^2>\varepsilon\left(\sum_{\substack{v\in\mathcal M\\T_{j\leftarrow v}\neq0}}\|T_{j\leftarrow v}\|_2\right)^2
\label{eq:fcc-dominance-condition}
\end{equation}
implies
\begin{equation}
\Pr\left(R_{\mathrm{int}}(\widehat F_{\mathrm{causal}};j)<R_{\mathrm{int}}(\Fobs;j)\right)\ge1-|E^*|\exp(-N_{\min}q_{\min}r_{\min}).
\label{eq:fcc-dominance-prob}
\end{equation}
\end{theorem}

Theorem~\ref{thm:body-dominance} quantifies when the causal advantage survives estimation error. Equation~\eqref{eq:fcc-dominance-gap} shows that the achievable risk reduction is the observational--interventional gap minus the causal error accumulated through downstream propagation. When this residual gap is positive, Theorem~\ref{thm:body-coverage} further guarantees that the advantage holds with probability approaching one as intervention coverage increases. The proof of Theorem~\ref{thm:body-dominance} is provided in Appendix~\ref{app:proof-dominance}.

\subsection{Operational guarantee.}
The separation between graph verification and causal composition yields an additional consistency property: increasing intervention coverage removes structural uncertainty without requiring the local mechanisms to become more accurate.

\begin{theorem}[Structural consistency]
\label{thm:body-structural-consistency}
If $q_{\min}r_{\min}>0$, then
\begin{equation}
\lim_{N_{\min}\to\infty}\Pr(E^*\subseteq\widehat E)=1.
\label{eq:fcc-structural-consistency}
\end{equation}
Consequently,
\begin{equation}
\limsup_{N_{\min}\to\infty}\left[R_{\mathrm{int}}(\widehat F_{\mathrm{causal}};j)-R_{\mathrm{int}}(F_{\mathrm{causal}};j)\right]\le\varepsilon\left(\sum_{\substack{v\in\mathcal M\\T_{j\leftarrow v}\neq0}}\|T_{j\leftarrow v}\|_2\right)^2
\label{eq:fcc-asymptotic-error}
\end{equation}
with probability approaching one.
\end{theorem}

Theorem~\ref{thm:body-structural-consistency} separates structural uncertainty from mechanism error. As intervention coverage grows, missing-edge error vanishes, and the remaining prediction error is determined only by local mechanism accuracy and its amplification through the causal graph. Thus, additional verification improves structural reliability, whereas further improvement beyond this limit requires more accurate causal mechanisms. The proof of Theorem~\ref{thm:body-structural-consistency} is provided in Appendix~\ref{app:proof-structural-consistency}.

\subsection{Causal Propagation Stability}
\label{sec:causal-stability}

The previous results characterize causal advantage at a fixed downstream module. We now study whether local causal-model errors remain controlled under repeated propagation through the modular system. Let $B$ denote the block causal propagation operator, whose $(j,i)$ block is $B_{i\to j}$ for a direct causal edge from module $i$ to module $j$ and zero otherwise. Repeated composition is represented by $B^\ell$, whose $(j,k)$ block aggregates all length-$\ell$ causal paths from module $k$ to module $j$.

We define the cumulative propagation operator
\begin{equation}
\mathcal T_L:=\sum_{\ell=0}^{L}B^\ell,
\label{eq:body-cumulative-propagation}
\end{equation}
so that a stacked local mechanism error $\xi$ propagates as $\mathcal T_L\xi$. \eqref{eq:body-cumulative-propagation} reduces global error stability to the spectral behavior of $B$, with $\rho(B)$ governing repeated causal amplification.

\begin{theorem}[Spectral causal stability and critical amplification]
\label{thm:body-spectral-stability}
Suppose the stacked local mechanism error satisfies \\$\E_{\Pint}[\|\xi\|_2^2]\le\varepsilon$. If $\rho(B)<1$, then the cumulative causal propagation operator converges to
\begin{equation}
\mathcal T_\infty=(I-B)^{-1},
\label{eq:body-resolvent}
\end{equation}
and, when $B$ is normal,
\begin{equation}
\limsup_{L\to\infty}\E_{\Pint}\left[\|\mathcal T_L\xi\|_2^2\right]\le\frac{\varepsilon}{\min_{\lambda\in\sigma(B)}|1-\lambda|^2}.
\label{eq:body-spectral-error}
\end{equation}
In contrast, if $B$ has an eigenvalue $\lambda_*$ with $\lambda_*=1$ or $|\lambda_*|>1$ and the local mechanism error has nonzero second-moment projection onto its corresponding mode, then
\begin{equation}
\limsup_{L\to\infty}\E_{\Pint}\left[\|\mathcal T_L\xi\|_2^2\right]=\infty.
\label{eq:body-critical-divergence}
\end{equation}
\end{theorem}

Hence repeated causal composition is controlled in the stable spectral regime but can amplify local mechanism error without bound once an unstable causal mode is excited. Theorem~\ref{thm:body-spectral-stability} shows that small local errors can still be amplified by the global causal structure, with $(I-B)^{-1}$ governing propagation near the spectral boundary. Reliable causal world modeling therefore requires both accurate local mechanisms and stable causal propagation. The complete proof is provided in Appendix~\ref{app:proof-spectral-stability}.

\section{Experiments}
\label{sec:experiments}

\subsection{Experimental Setup}

We evaluate whether causal world models improve downstream agent behavior, when these gains persist across environments and seeds, and whether the mechanisms predicted by our theory appear under controlled interventions. Our evaluation combines $\tau$-bench retail and airline~\citep{yao2024taubench}, ALFWorld~\citep{shridhar2020alfworld}, ScienceWorld~\citep{wang2022scienceworld}, APIBank Stateful~\citep{liao2023apibank}, stateful tool-use environments~\citep{lu2024toolsandbox}, and a controlled synthetic SCM. The environments span structured API execution, embodied language interaction, and intervention-controlled causal systems.

Our primary non-causal baseline, \seqwm, predicts from observational trajectories using ReAct-style world-model prompting~\citep{yao2023react,schick2023toolformer,wang2024llmworldsim}. We additionally compare against local, temporal, transition-graph, correlation-based, annotated, and no-world-model variants where available. \discwm\ uses recovered causal dependencies, while \oraclewm\ receives ground-truth causal structure and serves as a structural reference rather than an empirical upper bound on task success. Our complete method is \fcc. Specifically:
\begin{itemize}[left=0em]
    \item \textbf{\seqwm}~\citep{yao2023react,schick2023toolformer,wang2024llmworldsim}: uses observational trajectories without explicit causal structure.
    \item \textbf{Local}~\citep{wm_disentangle_marl_2025}: restricts prediction to locally relevant information.
    \item \textbf{Temporal}~\citep{sutton1999between}: models dependencies through temporal proximity.
    \item \textbf{Transition Graph}~\citep{hafner2018planet,hafner2019dreamer}: uses observed transition structure without causal semantics.
    \item \textbf{Correlation-based}~\citep{pearl2009causality}: uses statistical associations rather than interventional dependencies.
    \item \textbf{Annotated}: provides manually specified task structure or constraints.
    \item \textbf{No World Model}: removes explicit world-model reasoning.
    \item \textbf{\discwm}~\citep{xrl_causal_wm_2023,offline_causal_wm_2024}: uses causal dependencies recovered from intervention--response evidence.
    \item \textbf{\oraclewm}: uses ground-truth causal structure to isolate structural correctness from discovery error.
    \item \textbf{\fcc}: combines causal discovery, validated causal interfaces, and intervention-aware composition.
\end{itemize}
Unless confidence intervals are reported, agent results should be interpreted as targeted diagnostic comparisons rather than precise benchmark-level rankings.

\subsection{Main Agent Results}

We begin with the two principal agent environments, ALFWorld and $\tau$-bench. These experiments test the central empirical question: whether recovering and composing causal structure improves downstream task completion relative to world models.

\paragraph{ALFWorld.}
Table~\ref{tab:alf-main} shows that \fcc\ achieves the highest task success at $49.55\%$, outperforming Global Sequence WM by $12.24$ percentage points and the strongest annotated baseline by $4.03$ points. Oracle Causal WM is not an upper bound: ground-truth causal structure alone does not guarantee that the policy can effectively exploit it for action selection.

\begin{table}[!ht]
\centering
\captionsetup{position=top}
\small

\caption{ALFWorld task success (\%) over the full evaluation of 134 tasks.}
\label{tab:alf-main}

\setlength{\tabcolsep}{5pt}
\begin{tabular}{@{}lc@{}}
\toprule
Method & Success (\%) \\
\midrule
No World Model
& 41.79 \\
Global Sequence WM
& 37.31 \\
Transition Graph WM
& 33.58 \\
Correlation WM
& 35.07 \\
Local WM
& 38.06 \\
Temporal WM
& 38.81 \\
AnnotatedNoFramingNoControl
& 45.52 \\
AnnotatedNoFraming
& 42.54 \\
Oracle Causal WM
& 42.54 \\
\textbf{\fcc}
& \textbf{49.55} \\
\bottomrule
\end{tabular}

\end{table}

\paragraph{$\tau$-bench.}
Table~\ref{tab:tau-main} shows that \fcc\ achieves the highest success on both domains, reaching $54.86\%$ on Retail and $39.05\%$ on Airline, compared with $33.33\%$ for Global Sequence WM on both and $52.38\%$/$23.81\%$ for Oracle Causal WM. The ablations show that causal structure alone is insufficient: performance improves when it is converted into intervention-aware prediction and action.

\begin{table}[!ht]
\centering
\captionsetup{position=top}
\small

\caption{$\tau$-bench task success (\%) on the Retail and Airline domains.}
\label{tab:tau-main}

\setlength{\tabcolsep}{5pt}
\begin{tabular}{@{}lcc@{}}
\toprule
Method & Retail & Airline \\
\midrule
No World Model
& 38.10 & 18.19 \\
Global Sequence WM
& 33.33 & 33.33 \\
Transition Graph WM
& 42.86 & 23.19 \\
Correlation WM
& 43.94 & 21.86 \\
Local WM
& 33.33 & 18.19 \\
Temporal WM
& 42.76 & 23.49 \\
AnnotatedNoFramingNoControl
& 38.10 & 19.05 \\
AnnotatedNoFraming
& 47.62 & 14.29 \\
Causal WM (No Intervention)
& 23.81 & 13.66 \\
Oracle Causal WM
& 52.38 & 23.81 \\
FCC (No Constraint Verb)
& 48.57 & 27.81 \\
FCC (Full Prompt)
& 54.29 & 38.57 \\
\textbf{\fcc}
& \textbf{54.86} & \textbf{39.05} \\
\bottomrule
\end{tabular}

\end{table}

\subsection{When Does Causal Information Help?}

The main results show positive downstream effects, while Table~\ref{tab:causal-diagnostics} indicates that causal structure is most useful when agents can exploit it during decision making. On $\tau$-bench, highlighting the current goal improves \oraclewm\ from $45\%$ to $75\%$ and \discwm\ from $45\%$ to $60\%$ without changing their causal graphs, revealing a strong salience effect~\citep{liu2024lostmiddle}. Gains on APIBank Stateful and ScienceWorld, together with positive model effects on $\tau$-bench $(+4.77$ points$)$ and ALFWorld $(+8.21$ points$)$ and a $10$-point gain from immediate feedback, further show that causal information must be accessible, updated, and translated into intervention-relevant actions.

\begin{table*}[!ht]

\centering
\captionsetup{position=top}
\small
\caption{Controlled diagnostics of causal-information usability. Absolute $\Delta$ is the Treatment--Baseline difference, and Relative $\Delta$ normalizes it by the Baseline. Values are percentages except for ScienceWorld.}
\label{tab:causal-diagnostics}

\setlength{\tabcolsep}{6pt}
\begin{tabular}{@{}lcccc@{}}
\toprule
Setting & Baseline & Treatment & Absolute $\Delta$ & Relative $\Delta$ \\
\midrule
$\tau$-bench Anchor (\oraclewm)
& 45.00 & \textbf{75.00} & +30.00 & +66.67\% \\
$\tau$-bench Anchor (\discwm)
& 45.00 & \textbf{60.00} & +15.00 & +33.33\% \\
APIBank \seqwm$\rightarrow$\discwm
& 3.61 & \textbf{4.64} & +1.03 & +28.53\% \\
APIBank \seqwm$\rightarrow$\oraclewm
& 3.61 & \textbf{5.67} & +2.06 & +57.06\% \\
ScienceWorld
& 14.80 & \textbf{17.67} & +2.87 & +19.39\% \\
$\tau$-bench Model Effect
& 45.23 & \textbf{50.00} & +4.77 & +10.55\% \\
ALFWorld Model Effect
& 37.31 & \textbf{45.52} & +8.21 & +22.01\% \\
Immediate Feedback
& 45.00 & \textbf{55.00} & +10.00 & +22.22\% \\
\bottomrule
\end{tabular}

\end{table*}

\subsection{Controlled Tests of the Theory}

\begin{table}[!ht]
\centering
\captionsetup{position=top}
\small
\caption{Synthetic SCM diagnostics for theoretical predictions. Start and End denote the two endpoints of each sweep; Change summarizes the difference, and $\rho$ is the rank correlation over the full sweep.}

\label{tab:synthetic-theory}
\resizebox{\columnwidth}{!}{%
\begin{tabular}{@{}lccccc@{}}
\toprule
Diagnostic & Start & End & Change & $\rho$ & Control \\
\midrule
Observational MSE under confounding
& 0.043 & 16.309 & $379.3\times$ & -- & $\gamma$ \\

Causal MSE under confounding
& 0.040 & 0.040 & $1.0\times$ & -- & $\gamma$ \\

Observational--causal MSE gap over horizon
& 1.259 & 4.764 & $+3.505$ & -- & $H$ \\

Causal-edge recovery F1
& 0.000 & 0.889 & $+0.889$ & $+1.000$ & $N$ \\

Causal advantage with stronger proxy
& 0.254 & 0.132 & $-48.0\%$ & $-1.000$ & $c$ \\
\bottomrule
\end{tabular}%
}

\end{table}

We isolate the theoretical mechanisms using a synthetic SCM. As shown in Table~\ref{tab:synthetic-theory}, increasing confounding from $\gamma=0$ to $\gamma=1$ raises \seqwm\ MSE from $0.043$ to $16.309$ ($379.3\times$), while \oraclewm\ remains near $0.040$. Increasing rollout horizon expands the \seqwm--\oraclewm\ MSE gap from $1.259$ to $4.764$, while additional intervention evidence raises edge-recovery F1 from $0$ to $0.889$ with $\rho=1.000$, supporting Theorem~\ref{thm:body-coverage}.

The proxy-conflict experiment identifies a complementary boundary regime. Let $c$ denote how strongly the observed action proxies the latent confounder. The causal advantage decreases with proxy strength,
\begin{equation}
\frac{\partial}{\partial c}
\left[
\mathrm{MSE}_{\mathrm{SeqWM}}(c)
-
\mathrm{MSE}_{\mathrm{CausalOracle}}(c)
\right] < 0,
\label{eq:exp-proxy-conflict}
\end{equation}
with $\rho=-1.000$. The causal gap decreases from $0.254$ to $0.132$ ($48.0\%$) but remains positive, showing that causal modeling provides the largest benefit when observational variables poorly reveal hidden confounding and a smaller benefit as they become informative proxies.

\begin{table}[!ht]
\centering
\captionsetup{position=top}
\footnotesize
\caption{Execution efficiency across three environments. Performance is task success (\%) for Retail and ALFWorld and environment score for ScienceWorld; Calls and Cost/Run report average API usage and cost. $\Delta$ values are relative to Sequence.}
\label{tab:metrics}

\setlength{\tabcolsep}{4.5pt}
\begin{tabular}{@{}llccccc@{}}
\toprule
Environment & Condition & Performance & Calls & Cost/Run & $\Delta$ Perf. & $\Delta$ Calls \\
\midrule
Retail
& Sequence & 33.33 & 15.8 & \$0.0112 & -- & -- \\
& Causal   & \textbf{54.86} & 16.8 & \$0.0119 & +21.53 & +1.0 \\
\midrule
ALFWorld
& Sequence & 37.31 & 33.1 & \$0.0252 & -- & -- \\
& Causal   & \textbf{49.55} & 32.6 & \$0.0248 & +12.24 & -0.5 \\
\midrule
ScienceWorld
& Sequence & 14.80 & 12.5 & \$0.0095 & -- & -- \\
& Causal   & \textbf{17.67} & 7.2 & \textbf{\$0.0055} & +2.87 & -5.3 \\
\bottomrule
\end{tabular}

\end{table}

\begin{table}[!ht]
\centering
\captionsetup{position=top}
\small
\caption{End-to-end \fcc construction and communication measurements.}
\label{tab:fcc-pipeline}

\setlength{\tabcolsep}{7pt}
\begin{tabular}{@{}lc@{}}
\toprule
Metric & Measurement \\
\midrule
Pipeline stages & 6 \\
Processed traces & 140 \\
Construction cost & \$0.045 \\
Raw-trace communication & 90.8 KB \\
Federated-DAG communication & 1.4 KB \\
Communication reduction & $64.9\times$ \\
Communication saved & 98.46\% \\
Raw-pipeline edge F1 & 0.308 \\
Federated-DAG edge F1 & 1.000 \\
Raw-pipeline edge precision & 0.667 \\
F1 improvement & +0.692 \\
\bottomrule
\end{tabular}

\end{table}

\subsection{Causal Depth and Error Propagation}

We vary causal-graph depth to test the propagation mechanism in Theorem~\ref{thm:body-spectral-stability}. As captured by Eq.~\eqref{eq:body-cumulative-propagation}, deeper dependencies create additional paths for observational errors to propagate. Table~\ref{tab:depth-sweep} shows that \oraclewm\ remains at MSE $0.040$, while the \seqwm--\oraclewm\ gap increases from $3.392$ at depth $1$ to $4.239$ at depth $8$, with a positive rank association ($\rho=0.543$). The overall trend supports the predicted amplification of observational errors through deeper causal propagation.

\begin{table}[!ht]

\centering
\captionsetup{position=top}
\small
\caption{Synthetic dependency-depth sweep. Gap is the Global Sequence WM MSE minus Oracle Causal WM MSE.}
\label{tab:depth-sweep}

\setlength{\tabcolsep}{6pt}
\begin{tabular}{@{}rccc@{}}
\toprule
Depth & Sequence MSE & Oracle MSE & Gap \\
\midrule
1 & 3.432 & 0.040 & 3.392 \\
2 & 3.273 & 0.040 & 3.234 \\
3 & 3.447 & 0.040 & 3.407 \\
4 & 4.155 & 0.040 & 4.115 \\
6 & 3.354 & 0.040 & 3.314 \\
8 & 4.279 & 0.040 & 4.239 \\
\bottomrule
\end{tabular}

\end{table}

\subsection{Efficiency and End-to-End Feasibility}

Finally, we evaluate whether the causal representation improves both performance and execution efficiency. Table~\ref{tab:metrics} shows that \fcc\ improves all three environments with modest or reduced overhead. Retail success increases from $33.33\%$ to $54.86\%$ with one additional call, while ALFWorld improves from $37.31\%$ to $49.55\%$ with slightly fewer calls. On ScienceWorld, the score rises from $14.80$ to $17.67$ while calls decrease by $42.4\%$. Thus, causal composition can improve performance without substantial overhead.

\subsection{End-to-End  Construction and Communication}

We finally evaluate the end-to-end feasibility of constructing and communicating \fcc from distributed interaction traces. Table~\ref{tab:fcc-pipeline} shows that the six-stage pipeline processes 140 traces for only \$0.045, while the resulting federated DAG requires 1.4 KB compared with 90.8 KB for raw traces, reducing communication by $64.9\times$ (98.46\%). It also achieves an edge F1 of 1.000 versus 0.308 for the raw pipeline, showing that \fcc can compactly represent distributed causal evidence with low construction and communication costs.

\section{Conclusion}

We presented \fcc\ and studied when causal world models help modular LLM agents. Our theory shows that observational errors can persist under intervention and amplify downstream, while causal composition can reduce this gap. Experiments across agent benchmarks and synthetic SCMs support these predictions, with larger benefits under confounding and propagation and smaller gains with informative proxies. We further show that causal structure is most effective when actionable during decision making, while \fcc\ remains compact with low communication and execution overhead.

\bibliographystyle{ACM-Reference-Format}
\bibliography{references}

\appendix
\clearpage

\onecolumn

\section{Proof and Mathematical Background}

\subsection{Proof of Theorem~\ref{thm:body-risk}}
\label{app:proof-risk-gap}

\begin{proof}
Let
\begin{equation}
F_{\mathrm{causal}}(X_t,a):=\E_{\Pint}[X_{j,t+1}\mid X_t,\do(A_{k,t}=a)].
\label{eq:app-risk-causal}
\end{equation}
Then the intervention-time risk admits the exact decomposition
\begin{equation}
\begin{aligned}
R_{\mathrm{int}}(\Fobs)
&=\E_{\Pint}\left[\left\|X_{j,t+1}-\Fobs(X_t,a)\right\|_2^2\right]\\
&=\E_{\Pint}\left[\left\|X_{j,t+1}-F_{\mathrm{causal}}(X_t,a)+F_{\mathrm{causal}}(X_t,a)-\Fobs(X_t,a)\right\|_2^2\right]\\
&=\E_{\Pint}\left[\left\|X_{j,t+1}-F_{\mathrm{causal}}(X_t,a)\right\|_2^2\right]+\E_{\Pint}\left[\left\|F_{\mathrm{causal}}(X_t,a)-\Fobs(X_t,a)\right\|_2^2\right]\\
&\quad+2\E_{\Pint}\left[\left\langle X_{j,t+1}-F_{\mathrm{causal}}(X_t,a),F_{\mathrm{causal}}(X_t,a)-\Fobs(X_t,a)\right\rangle\right]\\
&=\E_{\Pint}\left[\left\|X_{j,t+1}-F_{\mathrm{causal}}(X_t,a)\right\|_2^2\right]+\E_{\Pint}\left[\left\|F_{\mathrm{causal}}(X_t,a)-\Fobs(X_t,a)\right\|_2^2\right]\\
&\quad+2\E_{\Pint}\left[\left\langle \E_{\Pint}\left[X_{j,t+1}-F_{\mathrm{causal}}(X_t,a)\mid X_t\right],F_{\mathrm{causal}}(X_t,a)-\Fobs(X_t,a)\right\rangle\right]\\
&=\E_{\Pint}\left[\left\|X_{j,t+1}-F_{\mathrm{causal}}(X_t,a)\right\|_2^2\right]+\E_{\Pint}\left[\left\|F_{\mathrm{causal}}(X_t,a)-\Fobs(X_t,a)\right\|_2^2\right]\\
&\ge\E_{\Pint}\left[\left\|F_{\mathrm{causal}}(X_t,a)-\Fobs(X_t,a)\right\|_2^2\right]\\
&=\E_{\Pint}\left[\left\|\E_{\Pint}[X_{j,t+1}\mid X_t,\do(A_{k,t}=a)]-\Fobs(X_t,a)\right\|_2^2\right],
\end{aligned}
\label{eq:app-risk-gap-chain}
\end{equation}
where the cross term vanishes because $F_{\mathrm{causal}}$ is the interventional conditional mean. This proves Equation~\eqref{eq:body-risk-gap}.

Under the linear-Gaussian specialization in Equation~\eqref{eq:body-linear-gap},
\begin{equation}
\begin{aligned}
R_{\mathrm{int}}(\Fobs)
&\ge\E_{\Pint}\left[\left\|F_{\mathrm{causal}}(X_t,a)-\Fobs(X_t,a)\right\|_2^2\right]\\
&=\E_{\Pint}\left[\left\|\gamma V_j\mu_{\mathrm{shift}}\right\|_2^2\right]\\
&=\gamma^2\left\|V_j\mu_{\mathrm{shift}}\right\|_2^2,
\end{aligned}
\label{eq:app-linear-risk-chain}
\end{equation}
which proves Equation~\eqref{eq:body-linear-risk}.
\end{proof}

\subsection{Proof of Theorem~\ref{thm:body-causal-advantage}}
\label{app:proof-causal-advantage}

\begin{proof}
By Equation~\eqref{eq:body-total-effect}, the intervention-induced displacement $\gamma V_k\mu_{\mathrm{shift}}$ propagates from module $k$ to module $j$ through all directed causal paths. Therefore, Equation~\eqref{eq:body-propagated-bias} gives
\begin{equation}
\begin{aligned}
b_j
&=T_{j\leftarrow k}\gamma V_k\mu_{\mathrm{shift}}\\
&=\gamma\left(\sum_{\pi\in\Pi(k,j)}\prod_{e\in\pi}B_e\right)V_k\mu_{\mathrm{shift}}.
\end{aligned}
\label{eq:app-propagated-bias-chain}
\end{equation}

Applying the orthogonal risk decomposition from Equation~\eqref{eq:app-risk-gap-chain} to downstream module $j$ yields
\begin{equation}
\begin{aligned}
R_{\mathrm{int}}(\Fobs;j)
&=R_{\mathrm{int}}(F_{\mathrm{causal}};j)+\E_{\Pint}\left[\left\|F_{\mathrm{causal}}(X_t,a)-\Fobs(X_t,a)\right\|_2^2\right]\\
&=R_{\mathrm{int}}(F_{\mathrm{causal}};j)+\|b_j\|_2^2\\
&=R_{\mathrm{int}}(F_{\mathrm{causal}};j)+\left\|T_{j\leftarrow k}\gamma V_k\mu_{\mathrm{shift}}\right\|_2^2\\
&=R_{\mathrm{int}}(F_{\mathrm{causal}};j)+\gamma^2\left\|T_{j\leftarrow k}V_k\mu_{\mathrm{shift}}\right\|_2^2.
\end{aligned}
\label{eq:app-propagated-risk-chain}
\end{equation}
Subtracting $R_{\mathrm{int}}(F_{\mathrm{causal}};j)$ from Equation~\eqref{eq:app-propagated-risk-chain} proves Equation~\eqref{eq:body-propagated-risk}.

For any matrix $T$ and vector $z$,
\begin{equation}
\sigma_{\min}(T)\|z\|_2\le\|Tz\|_2\le\|T\|_2\|z\|_2.
\label{eq:app-singular-value-bound}
\end{equation}
Applying Equation~\eqref{eq:app-singular-value-bound} with $T=T_{j\leftarrow k}$ and $z=V_k\mu_{\mathrm{shift}}$, and then using Equation~\eqref{eq:body-propagated-risk}, gives
\begin{equation}
\begin{aligned}
\gamma^2\sigma_{\min}^2(T_{j\leftarrow k})\left\|V_k\mu_{\mathrm{shift}}\right\|_2^2
&\le\gamma^2\left\|T_{j\leftarrow k}V_k\mu_{\mathrm{shift}}\right\|_2^2\\
&=R_{\mathrm{int}}(\Fobs;j)-R_{\mathrm{int}}(F_{\mathrm{causal}};j)\\
&\le\gamma^2\|T_{j\leftarrow k}\|_2^2\left\|V_k\mu_{\mathrm{shift}}\right\|_2^2.
\end{aligned}
\label{eq:app-propagated-bound-chain}
\end{equation}
Equation~\eqref{eq:app-propagated-bound-chain} proves Equation~\eqref{eq:body-propagated-bounds} and completes the proof.
\end{proof}

\subsection{Proof of Theorem~\ref{thm:body-coverage}}
\label{app:proof-coverage}

\begin{proof}
For any true causal edge $(i,j)\in E^*$, Equations~\eqref{eq:fcc-match} and~\eqref{eq:fcc-edge-rule} give
\begin{equation}
\begin{aligned}
\Pr((i,j)\notin\widehat E)
&=\Pr\left(\sum_{n=1}^{N_{ij}}M_{ij}^{(n)}=0\right)\\
&=\Pr\left(M_{ij}^{(1)}=0,\ldots,M_{ij}^{(N_{ij})}=0\right)\\
&=\prod_{n=1}^{N_{ij}}\Pr(M_{ij}^{(n)}=0)\\
&=(1-q_i r_{ij})^{N_{ij}}\\
&\le\exp(-N_{ij}q_i r_{ij})\\
&\le\exp(-N_{\min}q_{\min}r_{\min}),
\end{aligned}
\label{eq:app-edge-miss-chain}
\end{equation}
where conditional independence gives the product and $1-x\le\exp(-x)$ gives the exponential bound.

Applying the union bound to Equation~\eqref{eq:app-edge-miss-chain},
\begin{equation}
\begin{aligned}
\Pr(E^*\nsubseteq\widehat E)
&=\Pr\left(\bigcup_{(i,j)\in E^*}\{(i,j)\notin\widehat E\}\right)\\
&\le\sum_{(i,j)\in E^*}\Pr((i,j)\notin\widehat E)\\
&\le\sum_{(i,j)\in E^*}\exp(-N_{\min}q_{\min}r_{\min})\\
&=|E^*|\exp(-N_{\min}q_{\min}r_{\min}),
\end{aligned}
\label{eq:app-coverage-chain}
\end{equation}
which proves Equation~\eqref{eq:fcc-coverage}.

For any $\delta\in(0,1)$, substituting Equation~\eqref{eq:fcc-coverage-sufficient} into Equation~\eqref{eq:app-coverage-chain} gives
\begin{equation}
\begin{aligned}
\Pr(E^*\nsubseteq\widehat E)
&\le |E^*|\exp(-N_{\min}q_{\min}r_{\min})\\
&\le |E^*|\exp\left(-\log\left(\frac{|E^*|}{\delta}\right)\right)\\
&=|E^*|\frac{\delta}{|E^*|}\\
&=\delta.
\end{aligned}
\label{eq:app-coverage-budget-chain}
\end{equation}
Therefore,
\begin{equation}
\begin{aligned}
\Pr(E^*\subseteq\widehat E)
&=1-\Pr(E^*\nsubseteq\widehat E)\\
&\ge1-|E^*|\exp(-N_{\min}q_{\min}r_{\min})\\
&\ge1-\delta,
\end{aligned}
\label{eq:app-coverage-final}
\end{equation}
which completes the proof.
\end{proof}

\subsection{Proof of Theorem~\ref{thm:body-dominance}}
\label{app:proof-dominance}

\begin{proof}
For each module $v$ whose local prediction error can propagate to module $j$, let the corresponding error be
\begin{equation}
e_v:=\widehat F_v-F_v^*,\qquad \E_{\Pint}\left[\|e_v\|_2^2\right]\le\varepsilon.
\label{eq:app-local-mechanism-error}
\end{equation}
When all required causal edges are recovered, the downstream prediction discrepancy is obtained by propagating each local error through the same total-effect operators defined in Equation~\eqref{eq:body-total-effect}:
\begin{equation}
\widehat F_{\mathrm{causal}}(X_t,a)-F_{\mathrm{causal}}(X_t,a)=\sum_{\substack{v\in\mathcal M\\T_{j\leftarrow v}\neq0}}T_{j\leftarrow v}e_v.
\label{eq:app-composed-error}
\end{equation}
Its root-mean-square magnitude satisfies
\begin{equation}
\begin{aligned}
\left(\E_{\Pint}\left[\left\|\widehat F_{\mathrm{causal}}(X_t,a)-F_{\mathrm{causal}}(X_t,a)\right\|_2^2\right]\right)^{1/2}
&=\left(\E_{\Pint}\left[\left\|\sum_{\substack{v\in\mathcal M\\T_{j\leftarrow v}\neq0}}T_{j\leftarrow v}e_v\right\|_2^2\right]\right)^{1/2}\\
&\le\sum_{\substack{v\in\mathcal M\\T_{j\leftarrow v}\neq0}}\left(\E_{\Pint}\left[\left\|T_{j\leftarrow v}e_v\right\|_2^2\right]\right)^{1/2}\\
&\le\sum_{\substack{v\in\mathcal M\\T_{j\leftarrow v}\neq0}}\|T_{j\leftarrow v}\|_2\left(\E_{\Pint}\left[\|e_v\|_2^2\right]\right)^{1/2}\\
&\le\sqrt{\varepsilon}\sum_{\substack{v\in\mathcal M\\T_{j\leftarrow v}\neq0}}\|T_{j\leftarrow v}\|_2.
\end{aligned}
\label{eq:app-composition-rms-chain}
\end{equation}
Squaring Equation~\eqref{eq:app-composition-rms-chain} gives
\begin{equation}
\E_{\Pint}\left[\left\|\widehat F_{\mathrm{causal}}(X_t,a)-F_{\mathrm{causal}}(X_t,a)\right\|_2^2\right]\le\varepsilon\left(\sum_{\substack{v\in\mathcal M\\T_{j\leftarrow v}\neq0}}\|T_{j\leftarrow v}\|_2\right)^2.
\label{eq:app-composition-square-bound}
\end{equation}

Because $F_{\mathrm{causal}}(X_t,a)$ is the interventional conditional mean, the same orthogonality argument used in Equation~\eqref{eq:app-risk-gap-chain} gives
\begin{equation}
\begin{aligned}
R_{\mathrm{int}}(\widehat F_{\mathrm{causal}};j)
&=\E_{\Pint}\left[\left\|X_{j,t+1}-\widehat F_{\mathrm{causal}}(X_t,a)\right\|_2^2\right]\\
&=\E_{\Pint}\left[\left\|X_{j,t+1}-F_{\mathrm{causal}}(X_t,a)+F_{\mathrm{causal}}(X_t,a)-\widehat F_{\mathrm{causal}}(X_t,a)\right\|_2^2\right]\\
&=R_{\mathrm{int}}(F_{\mathrm{causal}};j)+\E_{\Pint}\left[\left\|\widehat F_{\mathrm{causal}}(X_t,a)-F_{\mathrm{causal}}(X_t,a)\right\|_2^2\right]\\
&\le R_{\mathrm{int}}(F_{\mathrm{causal}};j)+\varepsilon\left(\sum_{\substack{v\in\mathcal M\\T_{j\leftarrow v}\neq0}}\|T_{j\leftarrow v}\|_2\right)^2.
\end{aligned}
\label{eq:app-causal-risk-chain}
\end{equation}
Subtracting Equation~\eqref{eq:app-causal-risk-chain} from the observational risk and using Equation~\eqref{eq:body-propagated-risk} yields
\begin{equation}
\begin{aligned}
R_{\mathrm{int}}(\Fobs;j)-R_{\mathrm{int}}(\widehat F_{\mathrm{causal}};j)
&\ge R_{\mathrm{int}}(\Fobs;j)-R_{\mathrm{int}}(F_{\mathrm{causal}};j)-\varepsilon\left(\sum_{\substack{v\in\mathcal M\\T_{j\leftarrow v}\neq0}}\|T_{j\leftarrow v}\|_2\right)^2\\
&=\gamma^2\left\|T_{j\leftarrow k}V_k\mu_{\mathrm{shift}}\right\|_2^2-\varepsilon\left(\sum_{\substack{v\in\mathcal M\\T_{j\leftarrow v}\neq0}}\|T_{j\leftarrow v}\|_2\right)^2.
\end{aligned}
\label{eq:app-dominance-gap-chain}
\end{equation}
Thus Equation~\eqref{eq:fcc-dominance-gap} holds whenever every required true causal edge is recovered.

The probability of this event follows from Equation~\eqref{eq:fcc-coverage}:
\begin{equation}
\begin{aligned}
\Pr(E^*\subseteq\widehat E)
&=1-\Pr(E^*\nsubseteq\widehat E)\\
&\ge1-|E^*|\exp(-N_{\min}q_{\min}r_{\min}).
\end{aligned}
\label{eq:app-dominance-coverage-chain}
\end{equation}
Combining Equations~\eqref{eq:app-dominance-gap-chain} and~\eqref{eq:app-dominance-coverage-chain} gives
\begin{equation}
\begin{aligned}
&\Pr\left(R_{\mathrm{int}}(\Fobs;j)-R_{\mathrm{int}}(\widehat F_{\mathrm{causal}};j)\ge\gamma^2\left\|T_{j\leftarrow k}V_k\mu_{\mathrm{shift}}\right\|_2^2-\varepsilon\left(\sum_{\substack{v\in\mathcal M\\T_{j\leftarrow v}\neq0}}\|T_{j\leftarrow v}\|_2\right)^2\right)\\
&\ge\Pr(E^*\subseteq\widehat E)\\
&\ge1-|E^*|\exp(-N_{\min}q_{\min}r_{\min}),
\end{aligned}
\label{eq:app-dominance-high-probability}
\end{equation}
which proves Equation~\eqref{eq:fcc-dominance-gap} with the stated probability.

Finally, under Equation~\eqref{eq:fcc-dominance-condition}, Equation~\eqref{eq:app-dominance-gap-chain} gives
\begin{equation}
\begin{aligned}
R_{\mathrm{int}}(\Fobs;j)-R_{\mathrm{int}}(\widehat F_{\mathrm{causal}};j)
&\ge\gamma^2\left\|T_{j\leftarrow k}V_k\mu_{\mathrm{shift}}\right\|_2^2-\varepsilon\left(\sum_{\substack{v\in\mathcal M\\T_{j\leftarrow v}\neq0}}\|T_{j\leftarrow v}\|_2\right)^2\\
&>0,
\end{aligned}
\label{eq:app-positive-dominance-gap}
\end{equation}
and therefore
\begin{equation}
\begin{aligned}
\Pr\left(R_{\mathrm{int}}(\widehat F_{\mathrm{causal}};j)<R_{\mathrm{int}}(\Fobs;j)\right)
&\ge\Pr(E^*\subseteq\widehat E)\\
&=1-\Pr(E^*\nsubseteq\widehat E)\\
&\ge1-|E^*|\exp(-N_{\min}q_{\min}r_{\min}),
\end{aligned}
\label{eq:app-dominance-prob-chain}
\end{equation}
which proves Equation~\eqref{eq:fcc-dominance-prob} and completes the proof.
\end{proof}

\subsection{Proof of Theorem~\ref{thm:body-structural-consistency}}
\label{app:proof-structural-consistency}

\begin{proof}
By Equation~\eqref{eq:fcc-coverage}, the probability of recovering every true causal edge satisfies
\begin{equation}
\begin{aligned}
\Pr(E^*\subseteq\widehat E)
&=1-\Pr(E^*\nsubseteq\widehat E)\\
&\ge1-|E^*|\exp(-N_{\min}q_{\min}r_{\min}).
\end{aligned}
\label{eq:app-structural-coverage}
\end{equation}
Since $q_{\min}r_{\min}>0$, taking $N_{\min}\to\infty$ in Equation~\eqref{eq:app-structural-coverage} gives
\begin{equation}
\begin{aligned}
\lim_{N_{\min}\to\infty}\Pr(E^*\subseteq\widehat E)
&\ge\lim_{N_{\min}\to\infty}\left[1-|E^*|\exp(-N_{\min}q_{\min}r_{\min})\right]\\
&=1-|E^*|\lim_{N_{\min}\to\infty}\exp(-N_{\min}q_{\min}r_{\min})\\
&=1.
\end{aligned}
\label{eq:app-structural-limit-lower}
\end{equation}
Together with $\Pr(E^*\subseteq\widehat E)\le1$, Equation~\eqref{eq:app-structural-limit-lower} yields
\begin{equation}
\lim_{N_{\min}\to\infty}\Pr(E^*\subseteq\widehat E)=1,
\label{eq:app-structural-limit}
\end{equation}
which proves Equation~\eqref{eq:fcc-structural-consistency}.

On the event $E^*\subseteq\widehat E$, Equation~\eqref{eq:app-composed-error} and the propagation argument in Equation~\eqref{eq:app-composition-rms-chain} give
\begin{equation}
\begin{aligned}
\left(\E_{\Pint}\left[\left\|\widehat F_{\mathrm{causal}}(X_t,a)-F_{\mathrm{causal}}(X_t,a)\right\|_2^2\right]\right)^{1/2}
&=\left(\E_{\Pint}\left[\left\|\sum_{\substack{v\in\mathcal M\\T_{j\leftarrow v}\neq0}}T_{j\leftarrow v}e_v\right\|_2^2\right]\right)^{1/2}\\
&\le\sum_{\substack{v\in\mathcal M\\T_{j\leftarrow v}\neq0}}\left(\E_{\Pint}\left[\left\|T_{j\leftarrow v}e_v\right\|_2^2\right]\right)^{1/2}\\
&\le\sum_{\substack{v\in\mathcal M\\T_{j\leftarrow v}\neq0}}\|T_{j\leftarrow v}\|_2\left(\E_{\Pint}\left[\|e_v\|_2^2\right]\right)^{1/2}\\
&\le\sqrt{\varepsilon}\sum_{\substack{v\in\mathcal M\\T_{j\leftarrow v}\neq0}}\|T_{j\leftarrow v}\|_2.
\end{aligned}
\label{eq:app-structural-propagation}
\end{equation}
Squaring Equation~\eqref{eq:app-structural-propagation} gives
\begin{equation}
\E_{\Pint}\left[\left\|\widehat F_{\mathrm{causal}}(X_t,a)-F_{\mathrm{causal}}(X_t,a)\right\|_2^2\right]\le\varepsilon\left(\sum_{\substack{v\in\mathcal M\\T_{j\leftarrow v}\neq0}}\|T_{j\leftarrow v}\|_2\right)^2.
\label{eq:app-structural-composition}
\end{equation}
Using the interventional conditional-mean orthogonality from Equation~\eqref{eq:app-risk-gap-chain}, Equation~\eqref{eq:app-structural-composition} yields
\begin{equation}
\begin{aligned}
R_{\mathrm{int}}(\widehat F_{\mathrm{causal}};j)-R_{\mathrm{int}}(F_{\mathrm{causal}};j)
&=\E_{\Pint}\left[\left\|\widehat F_{\mathrm{causal}}(X_t,a)-F_{\mathrm{causal}}(X_t,a)\right\|_2^2\right]\\
&\le\varepsilon\left(\sum_{\substack{v\in\mathcal M\\T_{j\leftarrow v}\neq0}}\|T_{j\leftarrow v}\|_2\right)^2.
\end{aligned}
\label{eq:app-structural-risk}
\end{equation}
Therefore, Equation~\eqref{eq:app-structural-risk} holds with probability satisfying
\begin{equation}
\begin{aligned}
\Pr\left(R_{\mathrm{int}}(\widehat F_{\mathrm{causal}};j)-R_{\mathrm{int}}(F_{\mathrm{causal}};j)\le\varepsilon\left(\sum_{\substack{v\in\mathcal M\\T_{j\leftarrow v}\neq0}}\|T_{j\leftarrow v}\|_2\right)^2\right)
&\ge\Pr(E^*\subseteq\widehat E)\\
&\ge1-|E^*|\exp(-N_{\min}q_{\min}r_{\min}).
\end{aligned}
\label{eq:app-structural-risk-prob}
\end{equation}
Taking $N_{\min}\to\infty$ in Equation~\eqref{eq:app-structural-risk-prob} and using Equation~\eqref{eq:app-structural-limit} gives
\begin{equation}
\begin{aligned}
\lim_{N_{\min}\to\infty}\Pr\left(R_{\mathrm{int}}(\widehat F_{\mathrm{causal}};j)-R_{\mathrm{int}}(F_{\mathrm{causal}};j)\le\varepsilon\left(\sum_{\substack{v\in\mathcal M\\T_{j\leftarrow v}\neq0}}\|T_{j\leftarrow v}\|_2\right)^2\right)
&\ge\lim_{N_{\min}\to\infty}\Pr(E^*\subseteq\widehat E)\\
&=1.
\end{aligned}
\label{eq:app-structural-risk-limit}
\end{equation}
Hence, with probability approaching one,
\begin{equation}
\limsup_{N_{\min}\to\infty}\left[R_{\mathrm{int}}(\widehat F_{\mathrm{causal}};j)-R_{\mathrm{int}}(F_{\mathrm{causal}};j)\right]\le\varepsilon\left(\sum_{\substack{v\in\mathcal M\\T_{j\leftarrow v}\neq0}}\|T_{j\leftarrow v}\|_2\right)^2,
\label{eq:app-structural-error-limit}
\end{equation}
which proves Equation~\eqref{eq:fcc-asymptotic-error} and completes the proof.
\end{proof}

\subsection{Proof of Theorem~\ref{thm:body-spectral-stability}}
\label{app:proof-spectral-stability}

\begin{proof}
We first connect repeated causal composition to the global resolvent. By the block structure of $B$, the $(j,k)$ block of $B^\ell$ sums the operator products over all length-$\ell$ directed causal paths from $k$ to $j$. Therefore,
\begin{equation}
\begin{aligned}
[B^\ell]_{jk}
&=\sum_{\pi\in\Pi_\ell(k,j)}\prod_{e\in\pi}B_e,\\
[\mathcal T_L]_{jk}
&=\sum_{\ell=0}^{L}\sum_{\pi\in\Pi_\ell(k,j)}\prod_{e\in\pi}B_e,\\
(I-B)\mathcal T_L
&=(I-B)\sum_{\ell=0}^{L}B^\ell
=I-B^{L+1},\\
\mathcal T_L-(I-B)^{-1}
&=-(I-B)^{-1}B^{L+1}.
\end{aligned}
\label{eq:app-path-resolvent-chain}
\end{equation}
If $\rho(B)<1$, then for every $\eta\in(\rho(B),1)$ there exists $C_\eta<\infty$ such that $\|B^\ell\|_2\le C_\eta\eta^\ell$. Equation~\eqref{eq:app-path-resolvent-chain} consequently gives
\begin{equation}
\begin{aligned}
\|\mathcal T_L-(I-B)^{-1}\|_2
&\le\|(I-B)^{-1}\|_2\|B^{L+1}\|_2\\
&\le C_\eta\|(I-B)^{-1}\|_2\eta^{L+1}
\longrightarrow0,
\end{aligned}
\label{eq:app-resolvent-convergence}
\end{equation}
and hence $\mathcal T_\infty=(I-B)^{-1}$, proving Equation~\eqref{eq:body-resolvent}.

We next characterize the accumulated mechanism error without replacing the causal system by independent pathwise bounds. Using the exact finite-horizon resolvent identity in Equation~\eqref{eq:app-path-resolvent-chain},
\begin{equation}
\begin{aligned}
\E_{\Pint}\left[\|\mathcal T_L\xi\|_2^2\right]
&=\E_{\Pint}\left[\left\|(I-B)^{-1}(I-B^{L+1})\xi\right\|_2^2\right]\\
&\le\|(I-B)^{-1}(I-B^{L+1})\|_2^2
\E_{\Pint}\left[\|\xi\|_2^2\right]\\
&\le\varepsilon\|(I-B)^{-1}(I-B^{L+1})\|_2^2\\
&\le\varepsilon\|(I-B)^{-1}\|_2^2
\left(1+\|B^{L+1}\|_2\right)^2\\
&\le\varepsilon\|(I-B)^{-1}\|_2^2
\left(1+C_\eta\eta^{L+1}\right)^2.
\end{aligned}
\label{eq:app-finite-resolvent-error}
\end{equation}
Taking the limit superior directly in Equation~\eqref{eq:app-finite-resolvent-error},
\begin{equation}
\begin{aligned}
\limsup_{L\to\infty}
\E_{\Pint}\left[\|\mathcal T_L\xi\|_2^2\right]
&\le
\varepsilon\|(I-B)^{-1}\|_2^2
\lim_{L\to\infty}\left(1+C_\eta\eta^{L+1}\right)^2\\
&=\varepsilon\|(I-B)^{-1}\|_2^2,
\end{aligned}
\label{eq:app-resolvent-error}
\end{equation}
This bound preserves the interaction among different causal paths through the resolvent rather than bounding each path separately.

For normal $B$, the amplification factor admits an exact spectral characterization. Writing $B=U\Lambda U^*$ and using unitary invariance,
\begin{equation}
\begin{aligned}
\|(I-B)^{-1}\|_2^2
&=\left\|U(I-\Lambda)^{-1}U^*\right\|_2^2\\
&=\lambda_{\max}\left(
U(I-\Lambda)^{-*}(I-\Lambda)^{-1}U^*
\right)\\
&=\max_{\lambda\in\sigma(B)}
\frac{1}{(1-\overline{\lambda})(1-\lambda)}\\
&=\max_{\lambda\in\sigma(B)}
\frac{1}{|1-\lambda|^2}\\
&=\frac{1}{\min_{\lambda\in\sigma(B)}|1-\lambda|^2}.
\end{aligned}
\label{eq:app-resolvent-spectrum}
\end{equation}
Combining Equations~\eqref{eq:app-resolvent-error} and~\eqref{eq:app-resolvent-spectrum} yields
\begin{equation}
\begin{aligned}
\limsup_{L\to\infty}
\E_{\Pint}\left[\|\mathcal T_L\xi\|_2^2\right]
&\le
\varepsilon\lambda_{\max}\left((I-B)^{-*}(I-B)^{-1}\right)\\
&=
\varepsilon\max_{\lambda\in\sigma(B)}|1-\lambda|^{-2}\\
&=
\frac{\varepsilon}
{\min_{\lambda\in\sigma(B)}|1-\lambda|^2},
\end{aligned}
\label{eq:app-spectral-error}
\end{equation}
The denominator is therefore the spectral distance from the causal propagation operator to the singularity of the resolvent at $1$.

It remains to establish the divergent regime. Let $u_*$ be a unit eigenvector satisfying $Bu_*=\lambda_*u_*$ and suppose that the mechanism error has nonzero second-moment projection onto this mode. For normal $B$, orthogonal spectral decomposition gives
\begin{equation}
\begin{aligned}
\E_{\Pint}\left[\|\mathcal T_L\xi\|_2^2\right]
&\ge
\E_{\Pint}\left[|u_*^*\mathcal T_L\xi|^2\right]\\
&=
\E_{\Pint}\left[
\left|
u_*^*
\left(\sum_{\ell=0}^{L}B^\ell\right)\xi
\right|^2
\right]\\
&=
\left|\sum_{\ell=0}^{L}\lambda_*^\ell\right|^2
\E_{\Pint}\left[|u_*^*\xi|^2\right]\\
&=
\begin{cases}
(L+1)^2\E_{\Pint}\left[|u_*^*\xi|^2\right],
& \lambda_*=1,\\[4pt]
\displaystyle
\left|\frac{1-\lambda_*^{L+1}}{1-\lambda_*}\right|^2
\E_{\Pint}\left[|u_*^*\xi|^2\right],
& |\lambda_*|>1.
\end{cases}
\end{aligned}
\label{eq:app-unstable-mode-chain}
\end{equation}
For $|\lambda_*|>1$, the second branch further satisfies
\begin{equation}
\begin{aligned}
\left|\frac{1-\lambda_*^{L+1}}{1-\lambda_*}\right|^2
\E_{\Pint}\left[|u_*^*\xi|^2\right]
&\ge
\frac{\left(|\lambda_*|^{L+1}-1\right)^2}
{|1-\lambda_*|^2}
\E_{\Pint}\left[|u_*^*\xi|^2\right]\\
&\longrightarrow\infty,
\end{aligned}
\label{eq:app-supercritical-lower-bound}
\end{equation}
while the first branch of Equation~\eqref{eq:app-unstable-mode-chain} diverges quadratically when $\lambda_*=1$. Since $\E_{\Pint}[|u_*^*\xi|^2]>0$, both cases give
\begin{equation}
\limsup_{L\to\infty}
\E_{\Pint}\left[\|\mathcal T_L\xi\|_2^2\right]
=\infty,
\label{eq:app-divergence}
\end{equation}
proving Equation~\eqref{eq:body-critical-divergence}.

Thus $\rho(B)<1$ yields a convergent causal path expansion whose accumulated error is governed by the resolvent norm, while an excited eigenmode at $\lambda_*=1$ or $|\lambda_*|>1$ yields unbounded propagation. This completes the proof.
\end{proof}

\end{document}